\documentclass[letterpaper, 10 pt, conference]{ieeeconf}  

\IEEEoverridecommandlockouts                              
\usepackage{graphics} 
\usepackage{epsfig} 
\usepackage[noadjust]{cite}
\usepackage{amsmath}
\usepackage{mathrsfs}
\usepackage{graphicx}
\usepackage{amssymb}
\usepackage{colortbl}
\usepackage{arydshln}
\usepackage{stfloats}
\usepackage{longtable}
\usepackage{bm}

\usepackage[vlined,boxed,ruled,linesnumbered]{algorithm2e}
\SetKwProg{Fn}{Function}{}{}
\usepackage{algpseudocode}

\newtheorem{assumption}{Assumption}
\newtheorem{definition}{Definition}
\newtheorem{theorem}{Theorem}

\newtheorem{remark}{Remark}
\newtheorem{proposition}{Proposition}

\title{\LARGE \bf
Safety Verification of Neural Network Control Systems Using Guaranteed Neural Network Model Reduction
}

\author{Weiming Xiang and Zhongzhu Shao
\thanks{This research was supported by the National Science Foundation, under NSF CAREER Award 2143351 and NSF CNS Award 2223035.}
\thanks{Weiming Xiang is with School of Computer and Cyber Sciences, Augusta University, Augusta GA 30912, USA.
        {\tt\small wxiang@augusta.edu}}%
\thanks{Zhongzhu Shao is with School of Electrical Engineering, Southwest Jiaotong University, China.
        {\tt\small }}%
}

\begin{document}

\maketitle
\thispagestyle{empty}
\pagestyle{empty}

\begin{abstract}
This paper aims to enhance the computational efficiency of safety verification of neural network control systems by developing a guaranteed neural network model reduction method. First, a concept of model reduction precision is proposed to describe the guaranteed distance between the outputs of a neural network and its reduced-size version. A reachability-based algorithm is proposed to accurately compute the model reduction precision. Then, by substituting a reduced-size neural network controller into the closed-loop system, an algorithm to compute the reachable set of the original system is developed, which is able to support much more computationally efficient safety verification processes. Finally, the developed methods are applied to a case study of the Adaptive Cruise Control system with a neural network controller, which is shown to significantly reduce the computational time of safety verification and thus validate the effectiveness of the method.
\end{abstract}

\section{Introduction}

Neural networks are currently widely used in various fields, such as image processing \cite{LITJENS201760}, pattern recognition \cite{SCHMIDHUBER201585}, adaptive control \cite{HUNT19921083}, unmanned vehicles \cite{bojarski2016end} and aircraft collision avoidance systems \cite{Julian2016}, etc., demonstrating their powerful capabilities in solving complex and challenging problems that traditional approaches fail to address. 
As neural networks are further investigated, the size and complexity of their models continue to increase in order to improve their performance and accuracy to cope with complex and difficult tasks and changing environments. However, more complex large-scale neural network models also imply larger computational resources, such as larger memory, higher computational power and more energy consumption in applications \cite{Simon2020}. As a result, many neural network model reduction methods have been developed, such as parameter pruning and sharing, low-rank factorization, transfer/compact convolution filters, and knowledge distillation \cite{Yabo2019}. More results on neural network model reduction can be found in a recent survey \cite{deng2020model}.

On the other hand, due to the black-box nature of neural networks, neural networks are vulnerable in the face of resistance to interference/attacks. It has been observed that neural networks trained on large amounts of data are sometimes sensitive to updates and react to even small changes in parameters in unexpected and incorrect ways \cite{szegedy2013intriguing}. When neural networks are applied as controllers onto dynamical systems, they will inevitably suffer from safety problems due to the inevitable disturbances and uncertainties in the control process, further affecting the stability and safety of the whole closed-loop system. Therefore, when integrating neural networks into safety-critical control systems, the safety of the neural network needs to be guaranteed at all times, i.e., the safety verification of the neural network needs to be implemented. However, due to the sensitivity of neural networks to perturbations and the complex structure of neural networks, the verification of neural networks is extremely difficult. It has been demonstrated that the verification of simple properties of a small-scale neural network is an uncertainty polynomial (NP) complete problem \cite{katz2017reluplex}.
A few results have been reported in the literature for the formal verification of systems consisting of neural networks, readers are referred to the recent survey \cite{tran2020verification}. Specifically, reachability analysis is one of the promising safety verification tools such as in \cite{xiang2017reachable,xiang2020reachable,xiang2018output,tran2019star,tran2020nnv}, a simulation-based approach is proposed that transforms the difficulty of over-approximating the neural network's output set into a problem of estimating the neural network's maximal sensitivity, which is formulated as a series of convex optimization problems \cite{xiang2020reachable,xiang2018output}. Polytope-operation-based approaches were developed in \cite{xiang2017reachable,tran2019star,tran2020nnv} for dealing with a class of neural networks with activation functions of Rectified Linear Units (ReLU). However, the scalability issue is the major barrier preventing applying these methods to large-scale neural networks as well as neural network control systems active in a long period of time which means a large amount of reachable set computation is required during the time of interest.

In this paper, we propose a guaranteed model reduction method for neural network controllers based on the neural network reachability analysis and apply it to enhance the scalability of the reachability-based safety verification of closed-loop systems. Firstly, a concept of model reduction precision is proposed to accurately measure the distance between the outputs of an original neural network and its reduced-size version, and an approach to compute the model reduction precision is proposed, which ensures that the difference between the outputs obtained from two neural networks for a given input interval, chosen with any identical input, is within the model reduction precision. This algorithm is then applied to the model reduction of the neural network control system, enabling computationally efficient verification processes based on the reduced-size neural network controller. Finally, the correctness and feasibility of our approach are verified by applying it to the safety verification through the Adaptive Cruise Control (ACC) case study.

The remainder of the paper is organized as follows: Preliminaries are given in Section II. The guaranteed model reduction of neural networks is presented in Section III. The reachable set computation and safety verification algorithm for the neural network control system are presented in Section IV. The evaluation on the adaptive cruise control system is given in Section V. The conclusion is given in Section VI.

\section{Preliminaries}

In this paper, we consider a class of continuous-time nonlinear systems in the form of
\begin{align} \label{DynSys}
    \begin{cases}
    \mathbf{\dot x}(t) = f( \mathbf{x} (t), \mathbf{u}(t)) \\
    \mathbf{y}(t) = h( \mathbf{x}(t) )
    \end{cases} 
\end{align}
where $\mathbf{x}(t) \in \mathbb{R}^{n_x}$ is the state vector, $\mathbf{u}(t) \in \mathbb{R}^{n_u}$ is the control input and the $ \mathbf{y}(t) \in \mathbb{R}^{n_y}$ is the output vector. In general, the control input is in the form of
\begin{align} \label{ContIn}
    \mathbf{u}(t) = \gamma (\mathbf{y}(t), \mathbf{r}(t), t)
\end{align}
where $\mathbf{r}(t) \in \mathbb{R}^{n_r}$ is the reference input for the controller. 

To avoid the difficulties in the controller
design when system models are complex or even unavailable, one effective method is to use input-output data to train neural networks capable of generating appropriate control input signals to achieve control objectives. The neural network controller is in the form of
\begin{align} \label{ContIn_nn}
    \mathbf{u}(t) = \Phi (\mathbf{y}(t), \mathbf{r}(t))
\end{align}
where $\Phi$ denotes the neural network mapping output and reference signals to control input. 

In actual applications, the neural network receives input and generates output in a fraction of the computation time, so the control input generated by the neural network is generally discrete, generated only at each sampling time point $t_k, k \in \mathbb{N}$, and then remains a constant value between two successive sampling time instants. Therefore, the continuous-time nonlinear dynamical system with a neural network controller with sampling actions can be expressed in the following form of
\begin{align} \label{Dynsys_nn}
    {\begin{cases}
    \mathbf{\dot x}(t) = f( \mathbf{x} (t), \Phi ( \bm{\tau} (t_k))) \\
    \mathbf{y}(t) = h( \mathbf{x}(t) )
    \end{cases}} , \quad t \in [t_k, t_{k+1})
\end{align}
where $\bm{\tau}(t_k) = [\mathbf{y}^{\top}(t_k), \mathbf{r}^{\top}(t_k)]^{\top}$.

In this work, we consider feedforward neural networks for controllers in the form of  $\Phi: \mathbb{R}^{n_0} \to \mathbb{R}^{n_{L}}$ defined by the following recursive equations in the form of 
\begin{align}\label{eq:nn}
\begin{cases}
    \bm{\eta}_{\ell} = \phi_{\ell}(\mathbf{W}_{\ell} \bm{\eta}_{\ell-1}+\mathbf{b}_\ell),~\ell = 1,\ldots,L
    \\
    \bm{\eta}_{L} = \Phi(\bm{\eta}_0)
    \end{cases}
\end{align}
where $\bm{\eta}_\ell$ denotes the output of the $\ell$-th layer of the neural network, and in particular $\bm{\eta}_0\in\mathbb{R}^{n_0}$ is the input to the neural network and $\bm{\eta}_L\in \mathbb{R}^{n_L}$ is the output produced by the neural network, respectively. $\mathbf{W}_\ell \in \mathbb{R}^{n_{\ell}\times n_{\ell-1}}$ and $\mathbf{b}_{\ell} \in \mathbb{R}^{n_{\ell}}$ are weight matrices and bias vectors for the $\ell$-th layer.  $\phi_\ell = [\psi_{\ell},\cdots,\psi_{\ell}]$ is the concatenation of activation functions of the $\ell$-th layer in which $\psi_{\ell}:\mathbb{R} \to \mathbb{R}$ is the activation function. 

In this paper, we aim at reducing the computational cost of safety verification of neural network control systems in the framework of reachable set computation. 

\begin{definition} 
	Given a neural network in the form of (\ref{eq:nn}) and an input
	set $\mathcal{U}$, the following set
	\begin{align}
	\mathcal{Y} = \left\{\bm{\eta}_{L} \in \mathbb{R}^{n_L} \mid \bm{\eta}_{L} = \Phi (\bm{\eta}_{0}),~ \bm{\eta}_{0} \in \mathcal{U}\right\}
	\end{align}
	is called the output set of neural network (\ref{eq:nn}). 
\end{definition}

\begin{definition}
	A set $\mathcal{Y}_e$ is called an output reachable set over-approximation of neural network (\ref{eq:nn}), if $\mathcal{Y}\subseteq
	\mathcal{Y}_e$ holds, where $\mathcal{Y}$ is the output
	reachable set of neural network (\ref{eq:nn}).
\end{definition}

\begin{definition}
	Given a neural network control system in the form of (\ref{DynSys}) and (\ref{ContIn_nn})  with initial set $\mathcal{X}_0$ and input set $\mathcal{V}$, the reachable set at  time $t$ is
	\begin{equation}
	\mathcal{R}(t) =\left\{\mathbf{x}(t;\mathbf{x}_0,\mathbf{r}(\cdot))\in\mathbb{R}^{n_x} \mid \mathbf{x}_0 \in \mathcal{X}_0,~ \mathbf{r}(t) \in \mathcal{V}\right\}
	\end{equation}
	and the union of $\mathcal{R}(t)$ over $[t_0,t_f]$ defined by
	\begin{equation}
	\mathcal{R}([t_0,t_f]) = \bigcup\nolimits_{t \in [t_0,t_f]}\mathcal{R}(t) 
	\end{equation}
	is the reachable set over time interval $[t_0,t_f]$. 
\end{definition}

\begin{definition}\label{def2}
	A set $\mathcal{R}_e(t)$ is an over-approximation of $\mathcal{R}(t)$
	at time $t$ if $\mathcal{R}(t) \subseteq \mathcal{R}_e(t)$ holds.  Moreover, $\mathcal{R}_e([t_0,t_f]) = \bigcup\nolimits_{t \in[t_0,t_f]} \mathcal{R}_e(t)$ is an
	over-approximation of $\mathcal{R}([t_0,t_f])$ over time interval $[t_0, t_f]$.
\end{definition}

\begin{definition}
	Safety specification $\mathcal{S}$ formalizes the safety requirements for state $\mathbf{x}(t)$ of neural network control system (\ref{DynSys}), and is a predicate over state $\mathbf{x}(t)$ of neural network control system (\ref{DynSys}). The neural network control system (\ref{DynSys}) is safe over time interval $[t_0, t_f ]$ if the	following condition is satisfied:
	\begin{equation}\label{verification}
	\mathcal{R}_e([t_0,t_f])  \cap \neg \mathcal{S} = \emptyset
	\end{equation}
	where $\neg$ is the symbol for logical negation.
\end{definition}

As indicated in \cite{xiang2017reachable,xiang2020reachable,xiang2018output,tran2019star,tran2020nnv}, the computation cost for reachable set computation heavily relies on the size of neural networks, i.e., numbers of layers and neurons. In this paper, we aim to reduce the size of the neural network controller and rigorously compute the model reduction error, i.e., guaranteed neural network model reduction, so that the reachable set computation can be efficiently performed on a significantly reduced-size neural network and then mapped back to the original neural network to reach safety verification conclusions.

\section{Guaranteed Neural Network Model Reduction}

Given a large-scale neural network $\Phi$, there exist a large number of neural network model reduction methods as in survey paper \cite{deng2020model} to obtain its reduced-size version $\hat \Phi$ as below:
\begin{align}\label{eq:reduce_nn}
\begin{cases}
    \hat{\bm{\eta}}_{\ell} = \hat{\phi}_{\ell}(\hat{\mathbf{W}}_{\ell} \hat{\bm{\eta}}_{\ell-1}+\hat{\mathbf{b}}_\ell),~\ell = 1,\ldots,\hat L
    \\
    \hat{\bm{\eta}}_{L} = \hat{\Phi}(\hat{\bm{\eta}}_0)
    \end{cases}
\end{align}

To enable guaranteed neural network model reduction, the key is how to rigorously compute the output difference between the original neural network $\Phi$ and its reduced-size version $\hat \Phi$. Without loss of generality, the following assumption is given for neural network $\Phi$ and its reduced-size version $\hat \Phi$. 

\begin{assumption}\label{assump1}The following assumptions hold for neural network $\Phi$ and its reduced-size version $\hat \Phi$:
\begin{enumerate}
    \item The number of inputs of two neural networks are the same, i.e., $n_0 = \hat n_0$;
    \item The number of outputs of two neural networks are the same, i.e., $n_{L} =  n_{\hat L}$;
    \item The number of hidden layers of neural network  $\Phi$ is greater than or equal to the number of hidden layers of neural network  $\hat \Phi$, i.e., $L \ge \hat L$.
\end{enumerate}
\end{assumption}

To characterize the output difference between $\Phi$ and $\hat \Phi$, we define the following metric for model reduction precision.

\begin{definition}\label{def2} Consider neural network $\Phi$ and its reduced-size version $\hat \Phi$ with one same input set $\mathcal{U}$ and their corresponding output sets $\mathcal{Y}$ and $\hat{\mathcal{Y}}$, we define the distance between the outputs of $\Phi$ and $\hat \Phi$ with respect to the input set $\mathcal{U}$ by
\begin{align} \label{def2_1}
   \rho(\Phi,\hat\Phi,\mathcal{U}) = 
    \sup \limits_{\bm{\eta}_{0}=\hat{\bm{\eta}}_{0},\bm{\eta}_{0}, \hat{\bm{\eta}}_{0} \in \mathcal{U}} \left\|\Phi(\bm{\eta}_0) - \hat \Phi(\hat{\bm{\eta}}_0) \right\|
\end{align}
where $\rho(\Phi,\hat\Phi,\mathcal{U})$ is called model reduction precision. 
\end{definition}

In the framework of reachability analysis of neural networks, the following theorem presents a numerically tractable method to compute model reduction precision $\rho(\Phi,\hat\Phi,\mathcal{U})$.
\begin{theorem}\label{thm1}
Given neural network $\Phi$ and its reduced-size version $\hat \Phi$ with input set $\mathcal{U}$,  the model reduction precision $\rho(\Phi,\hat\Phi,\mathcal{U})$ can be computed by 
\begin{align}
    \rho(\Phi,\hat\Phi,\mathcal{U}) = \sup_{\tilde{\mathbf{\eta}}_0 \in {\mathcal{U}}}\left\|\tilde{\Phi}(\tilde{\bm{\eta}}_0)\right\|
\end{align}
where neural network $\tilde{\Phi}$ is an augmented neural network of $\Phi$ and $\hat{\Phi}$ defined as follows:
\begin{align}\label{eq:aug_nn}
\begin{cases}
    \tilde{\bm{\eta}}_{\ell} = \tilde{\phi}_{\ell}(\tilde{\mathbf{W}}_{\ell} \tilde{\bm{\eta}}_{\ell-1}+\tilde{\mathbf{b}}_\ell),~\ell = 1,\ldots,L+1
    \\
    \tilde{\bm{\eta}}_{L+1} = \tilde{\Phi}(\tilde{\bm{\eta}}_0)
    \end{cases}
\end{align}
in which input $\tilde{\bm{\eta}}_0 = \bm{\eta}_0 = \hat{\bm{\eta}}_0$ and
\begin{align} \label{eq:thm1_1}
    \tilde{\mathbf{W}}_{\ell} & = \begin{cases}
    \begin{bmatrix}
    \mathbf{W}_{1} \\ \hat{\mathbf{W}}_{1}
    \end{bmatrix}, &\ell = 1
    \\
    \begin{bmatrix}
    \mathbf{W}_{\ell} & \mathbf{0}_{{n_{\ell}} \times {\hat n_{\ell-1}}}
    \\
    \mathbf{0}_{{\hat n_{\ell}} \times {n_{\ell-1}}} & \hat{\mathbf{W}}_{\ell}
    \end{bmatrix}, &1 < \ell \le \hat L-1
    \\
    \begin{bmatrix}
    \mathbf{W}_{\ell} & \mathbf{0}_{{n_\ell} \times {\hat n_{\hat L-1}}}
    \\
    \mathbf{0}_{{\hat n_{\hat L-1}} \times {n_\ell}} & \mathbf{I}_{\hat n_{\hat L-1}}
    \end{bmatrix}, & \hat L \le \ell \le L-1
    \\
    \begin{bmatrix}
    \mathbf{W}_L & \mathbf{0}_{{n_L} \times {n_{\hat L-1}}}
    \\
    \mathbf{0}_{{n_{\hat L}} \times {n_{L-1}}} & \hat{\mathbf{W}}_{\hat L} 
    \end{bmatrix}, & \ell = L
    \\
    \begin{bmatrix}
    \mathbf{I}_{n_L} & -\mathbf{I}_{n_{\hat L}} 
    \end{bmatrix}, & \ell =L+1
    \end{cases}
    \end{align}
    
    \begin{align}
    \label{eq:thm1_2}
    \tilde{\mathbf{b}}_{\ell} &= \begin{cases}
    \begin{bmatrix}
    \mathbf{b}_\ell \\ 
    \hat{\mathbf{b}}_\ell
    \end{bmatrix}, & 1 \le \ell \le \hat L -1
    \\
    \begin{bmatrix}
    \mathbf{b}_\ell
    \\
    \mathbf{0}_{{\hat n_{\hat L-1}} \times 1}
    \end{bmatrix}, & \hat L \le \ell \le L-1 
    \\
    \begin{bmatrix}
    \mathbf{b}_L \\
    \hat{\mathbf{b}}_{\hat L}
    \end{bmatrix}, & \ell = L
    \\
    \begin{bmatrix}
    \mathbf{0}_{{(n_L+n_{\hat L})} \times 1} 
    \end{bmatrix}, & \ell = L+1
    \end{cases} 
    \\
    \label{eq:thm1_3}
    \tilde \phi_\ell(\cdot) &= \begin{cases}
    \begin{bmatrix}
    \phi_\ell(\cdot) \\   
    \hat \phi_\ell(\cdot)
     \end{bmatrix},& 1 \le \ell \le \hat L -1
    \\
     \begin{bmatrix}
      \phi_\ell(\cdot) \\  
      \mathsf{purelin}(\cdot)
    \end{bmatrix}, & \hat L \le \ell \le L-1 
    \\
    \    \begin{bmatrix}
    \phi_L(\cdot) \\   
    \hat \phi_{\hat L}(\cdot)
     \end{bmatrix}, & \ell = L
    \\
    \mathsf{purelin}(\cdot), & \ell = L+1
    \end{cases} 
\end{align}
where $\mathsf{purelin}(\cdot)$ is linear transfer function, i.e., $x=\mathsf{purelin}(x)$.
\end{theorem}
\begin{proof}
Given an input $\bm{\eta}_0 \in \mathcal{U}$ and $\tilde{\bm{\eta}}_0 = \hat{\bm{\eta}}_0 = \bm{\eta}_0$ and considering  layers $1 \le \ell \le \hat L-1$ of $\tilde \Phi$, we have
 \begin{align} \label{eq:pthm1_1}
     \tilde{\bm{\eta}}_{\hat L-1} 
     =
     \begin{bmatrix}
      \phi_{L-1}\circ \cdots \circ\phi_{1}( \mathbf{W}_{1}\bm{\eta}_{0}+ \mathbf{b}_1) \\  \hat \phi_{\hat L-1}\circ \cdots \circ\hat\phi_1(\hat{\mathbf{W}}_1\hat{\bm{\eta}}_0+ \hat{\mathbf{b}}_1)
     \end{bmatrix}
 \end{align}

Specifically, we consider $\ell = 1$ such that
 \begin{align} \label{eq:pthm1_2}
     \tilde{\bm{\eta}}_1 
     =
     \begin{bmatrix}
      \phi_1( \mathbf{W}_1\bm{\eta}_0+ \mathbf{b}_1) \\   
      \hat \phi_1(\hat{\mathbf{W}}_1\hat{\bm{\eta}}_0+ \hat{\mathbf{b}}_1)
     \end{bmatrix} =      \begin{bmatrix} 
     \bm{\eta}_1
       \\  
       \hat{\bm{\eta}}_1
     \end{bmatrix}
 \end{align}
 
Moreover, when $1 < \ell \le \hat L -1$, it leads to
\begin{align} \label{eq:pthm1_3}
     \tilde{\mathbf{W}}_{\ell} \tilde{\bm{\eta}}_{\ell-1}+ \tilde{\mathbf{b}}_{\ell} = \begin{bmatrix}
    \mathbf{W}_{\ell} {\bm{\eta}}_{\ell-1}+ {\mathbf{b}}_{\ell}
    \\
     \hat{\mathbf{W}}_{\ell} \hat{\bm{\eta}}_{\ell-1}+ \hat{\mathbf{b}}_{\ell}
    \end{bmatrix},~1 < \ell \le \hat L -1
 \end{align}
 
Staring from $\ell=1$ and recursively using (\ref{eq:pthm1_3}) into (\ref{eq:pthm1_1}), one can obtain
\begin{align}
     \tilde{\bm{\eta}}_{\hat L-1}=\begin{bmatrix} 
     \bm{\eta}_{\hat L-1}
       \\  
       \hat{\bm{\eta}}_{\hat L-1}
     \end{bmatrix}
\end{align}

Furthermore, when $\hat L \le \ell \le L-1$, one can derive
\begin{align}
    \tilde{\mathbf{W}}_\ell \tilde{\bm{\eta}}_{\ell-1} + \tilde{\mathbf{b}}_\ell&= \begin{bmatrix}
    \mathbf{W}_\ell\bm{\eta}_{\ell-1} + \mathbf{b}_{\ell}
    \\
    \hat{\bm{\eta}}_{\ell-1}
    \end{bmatrix},~\hat L \le \ell \le L-1
\end{align}
and 
 \begin{align} \nonumber
     \bm{\eta}_{L-1} &= \phi_{L-1}\circ \cdots \circ\phi_{\hat L}( \mathbf{W}_{\hat L}\bm{\eta}_{\hat L-1}+ \mathbf{b}_{\hat L-1})
 \end{align}
which leads to 
 \begin{align}
     \tilde{\bm{\eta}}_{L-1} 
     =&
     \begin{bmatrix} \bm{\eta}_{L-1}
       \\  
       \hat{\bm{\eta}}_{\hat L-1}
     \end{bmatrix}
 \end{align}
 
 Then, when $\ell=L$, it yields that
\begin{align}
    \tilde{\mathbf{W}}_L \tilde{\bm{\eta}}_{L-1} + \tilde{\mathbf{b}}_L &= \begin{bmatrix}
    \mathbf{W}_L\bm{\eta}_{L-1} + \mathbf{b}_L
    \\
     \hat{\mathbf{W}}_{\hat L}\hat{\bm{\eta}}_{\hat L-1}+\hat{\mathbf{b}}_{\hat L}
    \end{bmatrix}
 \end{align}

Thus, one can obtain
\begin{align}
     \tilde{\bm{\eta}}_{L} 
     =&
     \begin{bmatrix} \bm{\eta}_{L}
       \\  
       \hat{\bm{\eta}}_{\hat L}
     \end{bmatrix}
 \end{align}
 
 At last, when $\ell=L+1$, the following result can be obtained
 \begin{align} \nonumber
     \tilde{\bm{\eta}}_{L+1}=\tilde{\mathbf{W}}_{L+1} \tilde{\bm{\eta}}_L + \tilde{\mathbf{b}}_{L+1}  = \begin{bmatrix}
    \mathbf{I}_{n_L} & -\mathbf{I}_{n_{\hat L}} 
    \end{bmatrix} \begin{bmatrix} \bm{\eta}_{L}
       \\  
       \hat{\bm{\eta}}_{\hat L}
     \end{bmatrix}
 \end{align}
 which implies that
$
      \tilde{\bm{\eta}}_{L+1}=\bm{\eta}_{L}-\hat{\bm{\eta}}_{\hat L}
$.
Therefore, we can conclude that $\tilde{\Phi}(\tilde{\bm{\eta}}_0) =\Phi(\bm{\eta}_0) - \hat \Phi(\hat{\bm{\eta}}_0) $ as long as $\tilde{\bm{\eta}}_0 = \hat{\bm{\eta}}_0 = \hat \eta_0$ and 
\begin{align} \label{}
   \rho(\Phi,\hat\Phi,\mathcal{U}) = 
    \sup_{\tilde{\mathbf{\eta}}_0 \in {\mathcal{U}}}\left\|\tilde{\Phi}(\tilde{\bm{\eta}}_0)\right\|
\end{align}
The proof is complete.
\end{proof}

\begin{remark}
In the process of augmenting neural networks $\Phi$ and $\hat \Phi$ into $\tilde{\Phi}$, the case of $\ell=1$ in  (\ref{eq:thm1_1})--(\ref{eq:thm1_3}) ensures that  augmented neural network $\tilde{\Phi}$ takes the one same input $\bm{\eta}_{0}$ for the subsequent calls involving both processes of $\Phi$ and its reduced-size version $\hat \Phi$. Then, for $1 < \ell \le \hat L-1$, augmented neural network $\tilde{\Phi}$ conducts the computation of $\Phi$ and and its reduced-size version $\hat \Phi$ parallelly for the hidden layers of $1 < \ell \le \hat L-1$. When $\hat L \le \ell \le L-1$,  the hidden layers of reduced-size neural network $\hat \Phi$ which has fewer hidden layers are expanded to match the number of layers of the original neural network $\Phi$ with a larger number of hidden layers, but the expanded layers are forced to pass the information to subsequent layers without any changes, i.e., the weight matrices of the expanded hidden layers are identity matrices, and the bias vectors are zero vectors. This expansion is formalized as in the case of $\hat L \le \ell \le L-1$ in  (\ref{eq:thm1_1})--(\ref{eq:thm1_3}). Moreover, as $\ell=L$, this layer is a combination of output layers of both $\Phi$ and $\hat \Phi$ to generate the same outputs of $\Phi$ and $\hat \Phi$. At last, a comparison layer $L+1$ is added to compute the exact difference between the original neural network $\Phi$ and its reduced-size version $\hat \Phi$.
\end{remark}

\begin{remark}
As shown in Theorem \ref{thm1}, the key of computing model reduction precision $\rho(\Phi,\hat\Phi,\mathcal{U})$ is to compute the maximal output value of augmented neural network $\tilde{\Phi}$ with respect to input set $\mathcal{U}$. This can be efficiently done by neural network reachability analysis. For instance, as in NNV neural network reachability analysis tool, the reachable sets are in the form of a family of polyhedral sets \cite{tran2020nnv}, and in the IGNNV tool, the output reachable set is a family of interval sets \cite{xiang2018output,xiang2020reachable}. With the reachable set $\tilde{\mathcal{Y}}$, the model reduction precision $\rho(\Phi,\hat\Phi,\mathcal{U})$ can be easily obtained by searching for the maximal value of $\left\|\tilde{\bm{\eta}}_{L+1}\right\| $ in $\tilde{\mathcal{Y}}$, e.g., testing throughout a finite number of vertices in polyhedral sets.
\end{remark}

\section{Safety Verification of Neural Network Control Systems}

In this section, we apply neural network model reduction and model reduction precision to a neural network control system. By replacing the original neural network controller with a reduced-size neural network, the computational cost of safety verification can be significantly reduced. Moreover, the model reduction precision allows an over-estimation of the difference in behavior between the original neural network and the reduced-size one. For the reachability analysis of neural networks, the following result can be obtained. 

\begin{proposition}
\label{proposition_1}
Given neural network $\Phi$, its reduced-size version $\hat \Phi$  with output set $\hat{\mathcal{Y}}$, and model reduction precision $\rho(\Phi,\hat\Phi,\mathcal{U})$, the output reachable set of original neural network $\Phi$ satisfies
\begin{align} \label{eq:pro_1}
  \mathcal{Y} \subseteq \hat{\mathcal{Y}} \oplus \mathcal{B}(\mathbf{0}_{n_{L} \times 1},\frac{\rho(\Phi,\hat\Phi,\mathcal{U})}{2})
\end{align}
where $\mathcal{B}(\mathbf{0}_{n \times 1},r)$ denotes a ball centered at $\mathbf{0}_{n \times 1}$ with a radius of $r$, and $\oplus$ denotes the Minkowski sum.
\end{proposition}
\begin{proof}
This can be obtained straightforwardly by the definition of model reduction precision $\rho(\Phi,\hat\Phi,\mathcal{U})$ which characterizes the maximal difference between the outputs of $\Phi$ and $\hat \Phi$. The proof is complete. 
\end{proof}

\begin{algorithm}[b!]
\SetAlgoLined
\SetKwInOut{Input}{Input}
\SetKwInOut{Output}{Output}
\SetKw{Return}{return}
\Input{System dynamics $f$, $h$; Reduced-size neural network $\hat \Phi$; Model reduction precision $\rho(\Phi,\hat\Phi,\mathcal{U})$; Initial set $\mathcal{X}_0$; Input set $\mathcal{V}$}
\Output{Reachable set estimation $\mathcal{R}_e([t_0,t_f])$.}
\Fn{\texttt{reachNNCS}}{
\tcc{Initialization}
$k \gets 0$ \\
$t_{K+1}\gets t_f$ \\
$\mathcal{R}_e(t_0) \gets \mathcal{X}_0$ \\
\tcc{Iteration for all sampling intervals}
\While{$k \leq K$}{
$\mathcal{Y}_e(t_k) \gets \texttt{reachODEy}(h,\mathcal{R}_e(t_k))$ \\
$\mathcal{H} \gets \mathcal{Y}_e(t_k) \times \mathcal{V}$ \\
$\hat{\mathcal{U}}_e(t_k) \gets \texttt{reachNN}(\hat\Phi,\mathcal{H})$
\\
$\mathcal{U}_e \gets \hat{\mathcal{U}}(t_k) \oplus \mathcal{B}(\mathbf{0}_{n_{L} \times 1},\frac{\rho(\Phi,\hat\Phi,\mathcal{U})}{2})$
\\
$\mathcal{R}_e([t_k,t_{k+1}]) \gets \texttt{reachODEx}(f,\mathcal{U}_e,\mathcal{R}_e(t_k))$ \\
$k \gets k+1$\
}
\Return{$\mathcal{R}_e([t_0,t_f]) \gets \bigcup_{k=0,1\ldots,K}\mathcal{R}_e([t_k,t_{k+1}])$}  
}
 \caption{Reachable Set computation for Neural Network Control Systems (\ref{Dynsys_nn})} \label{alg1}
\end{algorithm}

The reachable set estimation for a sampled-data neural network control system in the form of (\ref{Dynsys_nn}) generally involves two parts: 1) Output set computation for neural network controllers denoted by 
\begin{align}
    \mathcal{Y} = \texttt{reachNN}(\Phi,\mathcal{U})
\end{align}
which can be efficiently obtained by neural network reachability tools such as \cite{xiang2020reachable,tran2020nnv} and (\ref{eq:pro_1}), and 2) Reachable set computation of system (\ref{DynSys}). For the reachable set computation of systems described by ODEs, there exist a variety of approaches and tools such as those well-developed in \cite{frehse2011spaceex,althoff2015introduction,chen2013flow,bak2017hylaa}. The following functions are given to denote the reachable set estimation for sampled data ODE models during $[t_k,t_{k+1}]$,
\begin{align}
\mathcal{R}_e([t_{k},t_{k+1}]) & = \texttt{reachODEx}(f,\mathcal{U}(t_k),\mathcal{R}_e(t_k))
	\\
\mathcal{Y}_e(t_k) & = \texttt{reachODEy}(h,\mathcal{R}_e(t_k))
\end{align}
where $\mathcal{U}(t_k)$ is the input set for sampling interval $[t_k,t_{k+1}]$. $\mathcal{R}_e(t_{k})$ and $\mathcal{R}_e([t_{k},t_{k+1}])$ are the estimated reachable sets for state $\mathbf{x}(t)$ at sampling instant $t_k$ and interval $[t_k,t_{k+1}]$, respectively. $\mathcal{Y}_e(t_k)$ is the estimated reachable set for output $\mathbf{y}(t_k)$. With the results in Proposition \ref{proposition_1}, we can use the reduced-size neural network to compute output set $\hat{\mathcal{U}}(t_k)$ which is much more computationally efficient due to its smaller size, and replace the output set of $\mathcal{U}(t_k)$ by $\hat{\mathcal{U}}(t_k) \oplus \mathcal{B}(\mathbf{0}_{n_{L} \times 1},\frac{\rho(\Phi,\hat\Phi,\mathcal{U})}{2})$ where $\rho(\Phi,\hat\Phi,\mathcal{U})$ is normally obtained through a one-time offline computation. The reachable set computation process is shown in Algorithm \ref{alg1}. 

Based on the estimated reachable set obtained by Algorithm \ref{alg1}, the safety property can be examined with the existence of intersections between the estimated reachable set and unsafe region $\neg \mathcal{S}$. 
	
\begin{proposition}\label{pro2}
Consider a neural network control system in the form of (\ref{Dynsys_nn}) with a safety specification $\mathcal{S}$, the system is safe in $[t_0,t_f]$, if $ \mathcal{R}_e([t_0,t_f]) \cap \neg \mathcal{S} = \emptyset$, where $\mathcal{R}_e([t_0,t_f])$ is an estimated reachable set obtained by Algorithm \ref{alg1}.
\end{proposition}

\section{Evaluation on Adaptive Cruise Control Systems}

In this section, our approach will be evaluated by the safety verification of an Adaptive Cruise Control (ACC) system equipped with a neural network controller as depicted in Fig. \ref{acc}. 
The system dynamics is in the form of
\begin{align}
    \begin{cases}
    \Dot{x}_l(t) = v_l (t) \\
    \Dot{v}_l(t) = \gamma_l(t) \\
    \Dot{\gamma}_l(t) = -2\gamma_l(t) + 2\alpha_l(t) - \mu v^2_l(t) \\
    \Dot{x}_e(t) = v_e (t) \\
    \Dot{v}_e(t) = \gamma_e(t) \\
    \Dot{\gamma}_e(t) = -2\gamma_e(t) + 2\alpha_e(t) - \mu v^2_e(t)
    \end{cases}
\end{align}
where $x_l(x_e)$, $v_l(v_e)$ and $\gamma_l(\gamma_e)$ are the position, velocity and actual acceleration of the lead (ego) car, respectively. $\alpha_l(\alpha_e)$ is the acceleration control input applied to the lead (ego) car, and $\mu = 0.001$ is the friction parameter. The ACC controller we considered here is a $5 \times 20$ feed-forward neural network with ReLU as its activation functions. The sampling scheme is considered as a periodic sampling every 0.01 seconds, i.e., $t_{k+1} - t_k = 0.01$ seconds.

\begin{figure}[b!]
\centering
	\includegraphics[width=8cm]{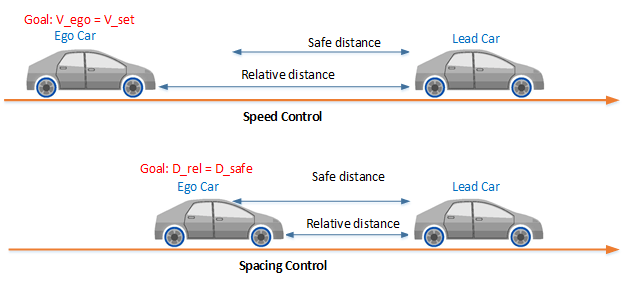}
	\caption{Adaptive cruise control system model \cite{xiang2020reachable}}
	\label{acc} 
\end{figure}

The sampled-data neural network controller for the acceleration control of the ego car is in the form of 
\begin{align}
    \alpha_e(t) = \Phi (v_{set}(t_k), t_{gap}, v_e(t_k), d_{rel}(t_k), v_{rel}(t_k))
\end{align}
in which $k \in [t_k, t_{k+1}]$.
The threshold of the safe distance between the two cars satisfies a function as defined below in the form of 
\begin{align}
    d_{safe} > d_{thold} = d_{def} + t_{gap} \cdot v_e
\end{align}
where $d_{safe}$ is the safe distance between the ego car and lead car, $d_{thold}$ is the threshold of the safe distance, $d_{def}$ is the standstill default spacing, and $t_{gap}$ is the time gap between the vehicles. The safety verification scenario we consider is that the lead car decelerates with $\alpha_l = 2$ to reduce its speed as an emergency braking occurs. We expect that the ego car guided by a neural network control system is able to maintain a safe relative distance to the lead car to avoid the collision. 

The safety specification parameter we consider in the simulation is $t_{gap} = 1.4$ seconds and $d_{def} = 10$. The time horizon that we want to verify is 3 seconds, i.e., 300 sampling intervals, after the emergency braking comes into play. The initial sets are $x_l(0) \in [94, 96]$, $v_l(0) \in [30, 30.2]$, $\gamma_l(0) = 0$, $x_e(0) \in [10, 11]$, $v_e(0) \in [30, 30.2]$, and $\gamma_e(0) = 0$. 

As mentioned above, the size of the hidden layer of the original neural network controller is $5 \times 20$, i.e., 5 layers with 20 neurons in each layer, Through neural network model reduction, we replace the original neural network controller with a reduced-size neural network of hidden layer size $2 \times 5$, i.e., 2 layers with 5 neurons in each layer, combined with a model reduction precision $\rho = 1.0234$. For
the continuous-time nonlinear dynamics, we use CORA \cite{althoff2015introduction} to do the reachability analysis for the time interval between
two sampling instants, and IGNNV in \cite{xiang2020reachable} is used for neural network reachability analysis. 

\begin{figure}[b!]
\centering
	\includegraphics[width=8.5cm]{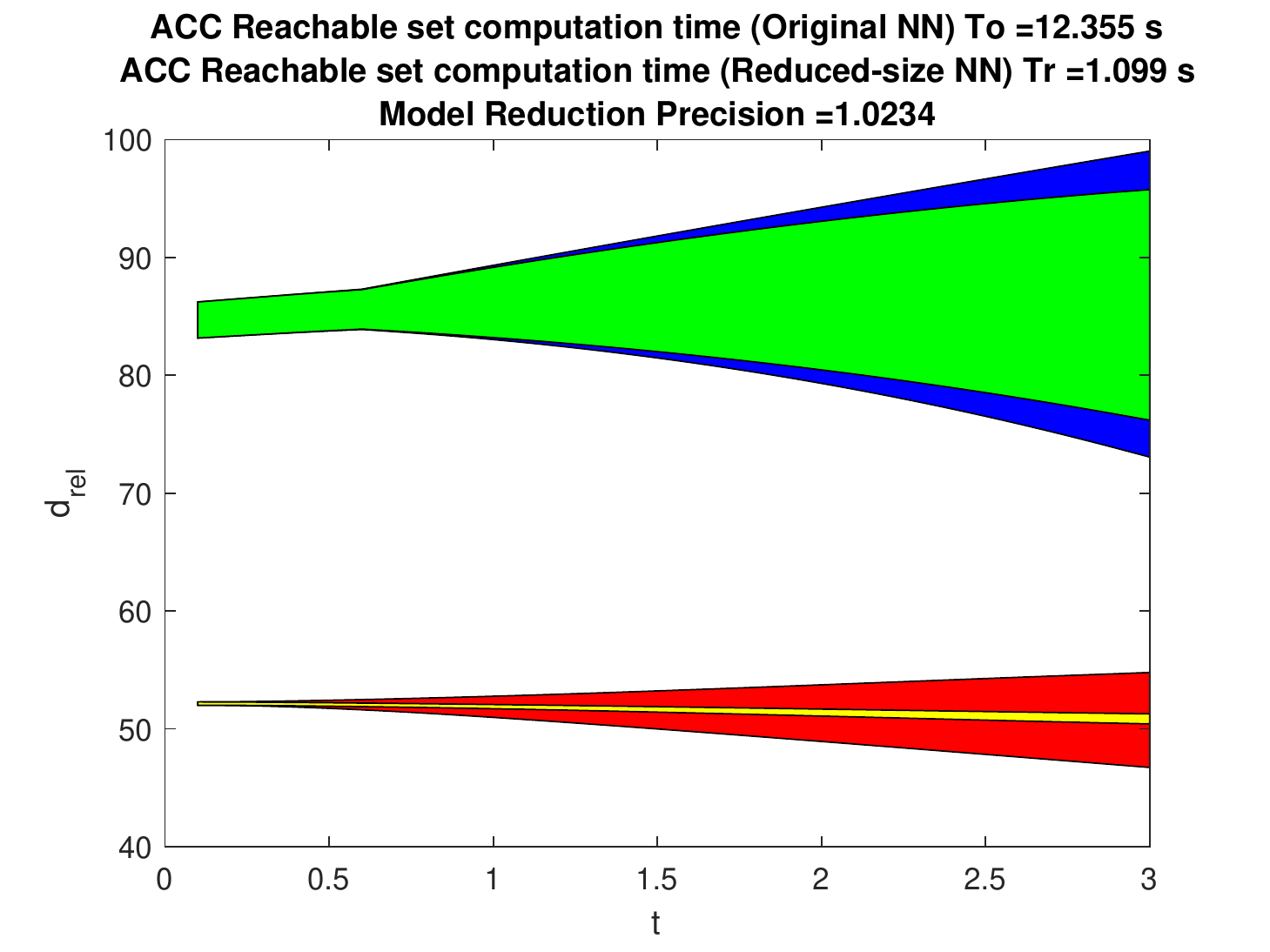}
	\caption{Reachable set of ACC systems. The green (blue) pipe indicates the output reachable set of the dynamic system with the original (reduced-size) neural network. It is obvious that the blue pipe completely wraps around the green pipe. The yellow (red) pipe denotes the safe distance threshold region for the dynamical system with the original (reduced-size) neural network.}
	\label{acc_reach} 
\end{figure}

\begin{figure}[t!]
\centering
	\includegraphics[width=8.5cm]{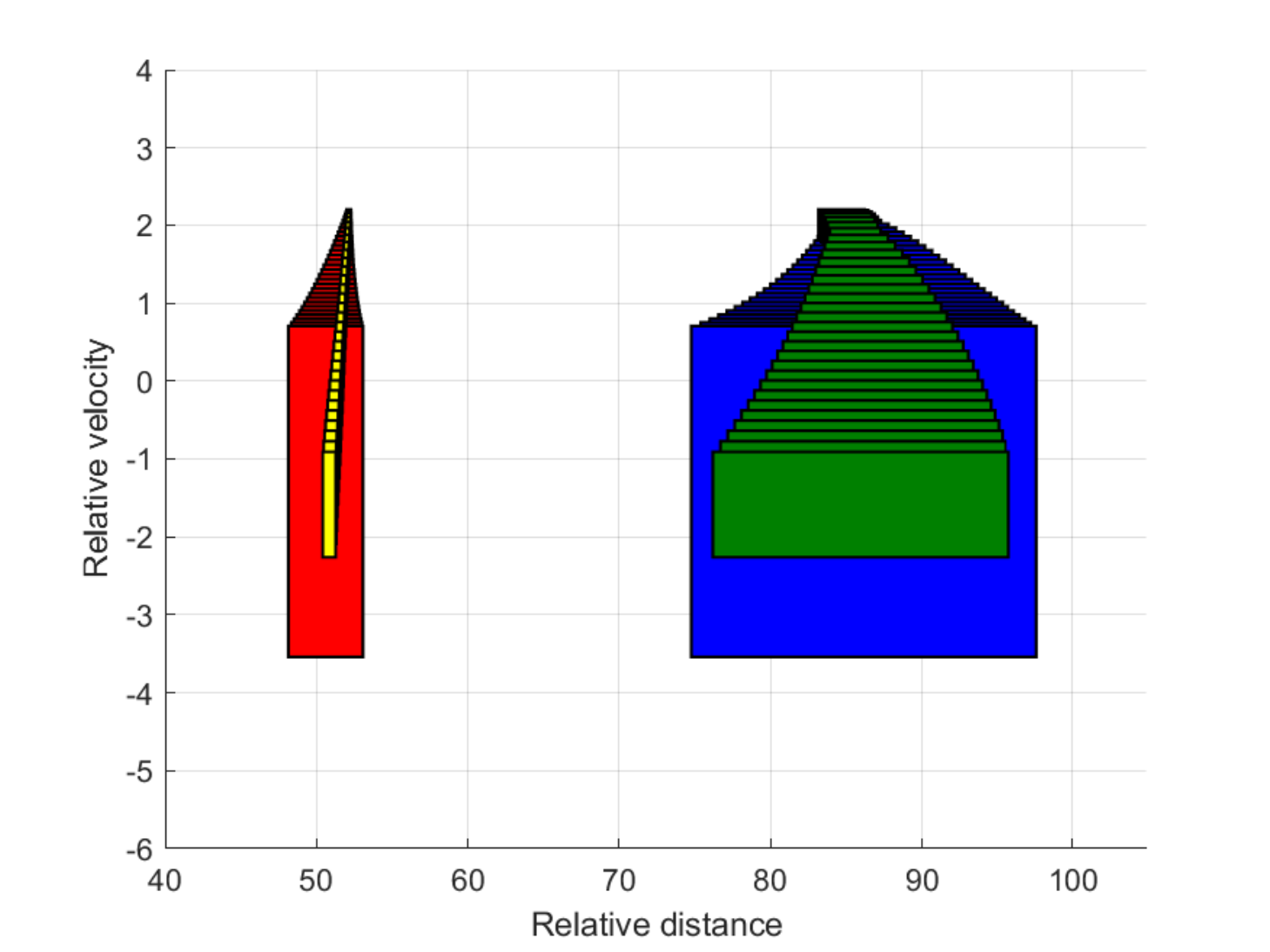}
	\caption{Reachable set estimation for relative distance and the relative velocity between lead and ego cars. The green (blue) pipe indicates the reachable set of the dynamical system with the original (reduced-size) neural network. The yellow (red) pipe denotes the safe distance threshold and the relative velocity between lead and ego cars for the dynamical system with the original (reduced-size) neural network.}
	\label{acc_reach_2d} 
\end{figure}

\begin{table}[t!]
	\centering
	\caption{Comparison of ACC reachable set calculation times}\label{tabl}
	\begin{tabular}{|cc|}
		\hline 
\textbf{Dynamical Systems} & \textbf{Computational Time}\\
		\hline\hline
ACC with original neural network    &$12.355$ s  \\
		\hline
ACC with reduced-size neural network  &$1.099$ s\\

\hline
	\end{tabular}
\end{table}

The output reachable set of the ACC system for the relative distance between the lead car and the ego car over time can be shown in Figs. \ref{acc_reach} and \ref{acc_reach_2d}. Notably, the computation time has been significantly reduced from 12.355 seconds to 1.099 seconds when using a reduced-size neural network as shown in Fig. \ref{acc_reach} and Table \ref{tabl}. The system is safe when the output reachable set of relative distances does not intersect with the safe distance threshold region.

In summary, the simulations show that the closed-loop system with the reduced-size neural network can be used for safety verification of the original system as long as the model reduction precision can be provided. The reduced-size neural network can significantly reduce the computational time of the entire process of solving the neural network control system for the output reachable set.

\section{Conclusions}
This paper investigates the problem of simplifying the safety verification of neural network control systems, proposes a concept of model reduction precision that characterizes the minimum upper bound on the outputs between a neural network and its reduced-size one, and proposes an algorithm to calculate the model reduction precision. By using a reduced-size neural network as the neural network controller and introducing the model reduction precision in the computation of the output reachable. Combined with the calculation of reachable sets for dynamical systems, we give the reachable set computation algorithm based on model reduction of neural network control systems. In this way, we can obtain the over-approximated output reachable set of the original neural network control system with less computation time and enable a simplification of safety verification processes. The developed results are applied to the ACC system to verify its effectiveness and feasibility.

\bibliographystyle{IEEEtran}

\bibliography{reference.bib}

\end{document}